\documentclass{article}
\usepackage{arxiv}

\usepackage[utf8]{inputenc} 
\usepackage[T1]{fontenc}    
\usepackage{hyperref}       
\usepackage{url}            
\usepackage{booktabs}       
\usepackage{amsfonts}       
\usepackage{nicefrac}       
\usepackage{microtype}      
\usepackage{lipsum}
\usepackage{graphicx}
\usepackage{amsmath}
\usepackage{booktabs}
\usepackage{multirow}
\usepackage{amsthm}

\newtheorem{theorem}{Theorem}

\theoremstyle{remark}
\newtheorem*{remark}{Remark}

\title{Lower Ricci Curvature for Hypergraphs}

\author{
 Shiyi Yang$^1$, Can Chen$^{1,2,3,\dagger}$, Didong Li$^{1,4,\dagger}$ \\
  Department of Biostatistics$^1$, School of Data Science and Society$^2$, Department of Mathematics$^3$, \\Department of Statistics and Operations Research$^4$, University of North Carolina at Chapel Hill\\
  $\dagger$: To whom correspondence should be addressed.
}

\begin{document}

\maketitle
\begin{abstract}
Networks with higher-order interactions,  prevalent in biological, social, and information systems, are naturally represented as hypergraphs, yet their structural complexity poses fundamental challenges for geometric characterization. While curvature-based methods offer powerful insights in graph analysis, existing extensions to hypergraphs suffer from critical trade-offs: combinatorial approaches such as Forman-Ricci curvature capture only coarse features, whereas geometric methods like Ollivier-Ricci curvature offer richer expressivity but demand costly optimal transport computations. To address these challenges, we introduce hypergraph lower Ricci curvature (HLRC), a novel curvature metric defined in  closed form that achieves a principled balance between interpretability and efficiency. Evaluated across diverse synthetic and real-world hypergraph datasets, HLRC consistently reveals meaningful higher-order organization, distinguishing intra- from inter-community hyperedges, uncovering latent semantic labels, tracking temporal dynamics, and supporting robust clustering of hypergraphs based on global structure. By unifying geometric sensitivity with algorithmic simplicity, HLRC provides a versatile foundation for hypergraph analytics, with broad implications for tasks including node classification, anomaly detection, and generative modeling in complex systems.
\end{abstract}

\section{Introduction}

Numerous real-world systems exhibit higher-order interactions that transcend pairwise relationships. In social networks, individuals often engage in group activities\cite{iacopini2024temporal,zhu2018social,zlatic2009hypergraph}; in co-citation and co-authorship networks, connections naturally form among multiple papers or researchers\cite{wu2022hypergraph,lerner2025relational}; and in protein-protein interaction networks, functional behavior frequently involves complexes of more than two proteins\cite{xia2024integration,zhang2024hierarchical}. These higher-order relationships are more effectively modeled by hypergraphs, a generalization of  traditional graphs by allowing each hyperedge to connect an arbitrary number of nodes\cite{berge1984hypergraphs,bretto2013hypergraph}. Hypergraphs provide an unambiguous representation, enabling more faithful modeling of systems in which collective behavior or shared context is fundamental, and where information flow, influence, or function emerges from simultaneous group participation\cite{chen2020tensor,chen2021controllability}. Understanding the geometry of hypergraphs, including notions like curvature, centrality, and higher-order connectivity, therefore becomes  critical for uncovering the latent organizational principles that govern such systems\cite{bryant2014diversities}. These geometric insights can reveal modular structures, bottlenecks, hierarchies, and redundancies that go beyond pairwise interactions, all of which inform the design of more efficient and interpretable algorithms for learning, inference, and prediction in high-dimensional, structured data environments\cite{zhu2020graph,papadopoulos2015network,monti2017geometric}.

The notion of curvature, originally formulated in the context of smooth manifolds in differential geometry\cite{federer1959curvature}, has been applied to discrete structures to quantify their “shape” and connectivity at multiple scales.  In particular, graph curvature measures have emerged as powerful tools for characterizing how the local geometry of a graph deviates from flatness, thereby capturing nuanced structural patterns in node connectivity and edge relationships\cite{ye2019curvature}. Among these, Ollivier-Ricci curvature\cite{ollivier2009ricci, ollivier2010survey, lott2009ricci, lin2011ricci} draws from optimal transport theory in Riemannian geometry to assess how metric distances between probability distributions over adjacent nodes reflect the network’s global organization. On the other hand, Forman-Ricci curvature\cite{fesser2024mitigating, sreejith2016forman} provides a combinatorial approach that evaluates the imbalance in local edge weights and degrees, yielding interpretable insights into node centrality, network robustness, and signal propagation.  These notions of discrete curvature have demonstrated utility across a broad spectrum of applications, including community detection\cite{sia2019ollivier, park2024lower,tian2025curvature,tian2023mixed}, robustness evaluation in cancer networks\cite{sandhu2015graph,pouryahya2018characterizing,elkin2021geometric,farooq2019network}, vulnerability analysis in infrastructure systems\cite{gao2019measuring,chen2024vulnerability,tan2024analyzing,jonckheere2019ollivier}, bottleneck identification in graph neural networks\cite{topping2021understanding,li2022curvature,sun2022self,liu2023curvdrop} and etc\cite{ni2015ricci, samal2018comparative,simhal2025orco,sia2022inferring}. Despite these advances, extending curvature concepts to hypergraphs poses significant theoretical and computational challenges, as the combinatorial and topological complexity of higher-order interactions breaks many of the assumptions underlying classical graph-theoretic definitions.

Recent efforts have sought to generalize curvature notions to the hypergraph setting, yielding two primary formulations (see Supplementary Note 1). Hypergraph Ollivier-Ricci curvature (HORC\cite{eidi2020ollivier, coupette2023ollivier}) extends its graph counterpart by formulating curvature as a multi-marginal optimal transport problem over the sets of nodes participating in each hyperedge. While theoretically elegant, this approach is computationally intensive: the dimensionality of the transport problem grows rapidly with hyperedge cardinality, resulting in prohibitive runtime and memory requirements that limit its scalability to real-world systems. In addition,  hypergraph Forman-Ricci curvature (HFRC\cite{leal2021forman}) offers a closed-form, combinatorial measure that can be efficiently computed for each hyperedge. However, this simplicity comes at the cost of expressiveness: HFRC relies solely on the degrees of nodes within a hyperedge, disregarding the broader connectivity patterns to adjacent hyperedges. As a result, it fails to capture important structural roles, such as distinguishing bridging hyperedges that connect disparate communities from those embedded within cohesive modules. Moreover, the curvature values produced often span large negative ranges with no intrinsic reference scale, limiting interpretability and comparative analysis across a network. These limitations highlight a pressing need for a hypergraph curvature measure that balances expressiveness and computational efficiency, enabling principled structural analysis in higher-order networked systems.

To address these limitations, we introduce  hypergraph lower Ricci curvature (HLRC), a novel extension of lower Ricci curvature for graphs\cite{park2024lower}, that captures higher-order connectivity through an effective, scalable manner. HLRC offers a closed-form, computationally efficient measure that quantifies the relational strength between node pairs by counting their co-memberships across hyperedges, with each contribution appropriately normalized to reflect interaction intensity in multi-node contexts. We evaluate HLRC on special structured hypergraphs, synthetic stochastic-block-model hypergraphs, and five real-world datasets, demonstrating that it consistently outperforms existing formulations HORC and HFRC in distinguishing community-like from bridge-like hyperedges, enhancing community detection in node-labeled settings, revealing interpretable curvature patterns in hyperedge-labeled networks, and supporting robust hypergraph clustering. These results highlight HLRC as a versatile geometric descriptor that advances structural understanding of complex higher-order networks, while providing a robust foundation for a wide range of downstream tasks, including node classification, anomaly detection, and generative modeling, across domains spanning biology, medicine, and social systems.

\section{Results}
\subsection{Characterizing hypergraph topology through curvature.}

HLRC is formulated for unweighted, undirected hypergraphs and integrates the geometric sensitivity of Ollivier-Ricci curvature with the computational simplicity of Forman’s method, yielding a interpretable and scalable measure of higher-order structure. For a given hyperedge $e$ of degree greater than two, HLRC is defined as
\begin{equation}\label{eq:hlrc}
    \text{HLRC}(e) = \sum_{v\in e} \frac{1}{n_v} + \frac{n_e+d_e/2-1}{\max_{v\in e} n_v} + \frac{n_e+d_e/2-1}{\min_{v\in e} n_v} - 1,
\end{equation}
where $n_v$ is the number of neighbors of node $v$, $n_e$ is the number of neighbors of hyperedge $e$, and $d_e$ is the degree (cardinality) of hyperedge $e$ (see ``Hypergraphs'' in Methods). This formulation  integrates two key components to capture a nuanced balance between a hyperedge’s internal and external connectivity contributions to curvature. First, the local node contribution, expressed as $\sum_{v \in e} \frac{1}{n_v}$, aggregates the inverses of the neighborhood sizes of nodes in $e$, thereby capturing how local connectivity influences curvature.  Second, hyperedge-level adjustments are captured by the terms $\frac{n_e+d_e/2-1}{\max_{v\in e} n_v}$ and $\frac{n_e+d_e/2-1}{\min_{v\in e} n_v}$, which jointly reflect the broader connectivity of the hyperedge by incorporating its neighborhood size and degree, normalized by the largest and smallest node neighborhood sizes within the hyperedge. The closed-form nature of Equ. \eqref{eq:hlrc} renders HLRC both analytically tractable and computationally efficient. Importantly, HLRC inherits desirable universal boundedness from its graph-based analogue, being provably bounded between $-1$ and 1 (see Supplementary Note 2). Values approaching 1 are indicative of tightly knit, clique-like structures, whereas values near $-1$ suggest structural bottlenecks or sparse interconnectivity, such as bridging hyperedges between otherwise disconnected regions of the network. 

\begin{figure}
    \centering
    \includegraphics[width=\linewidth]{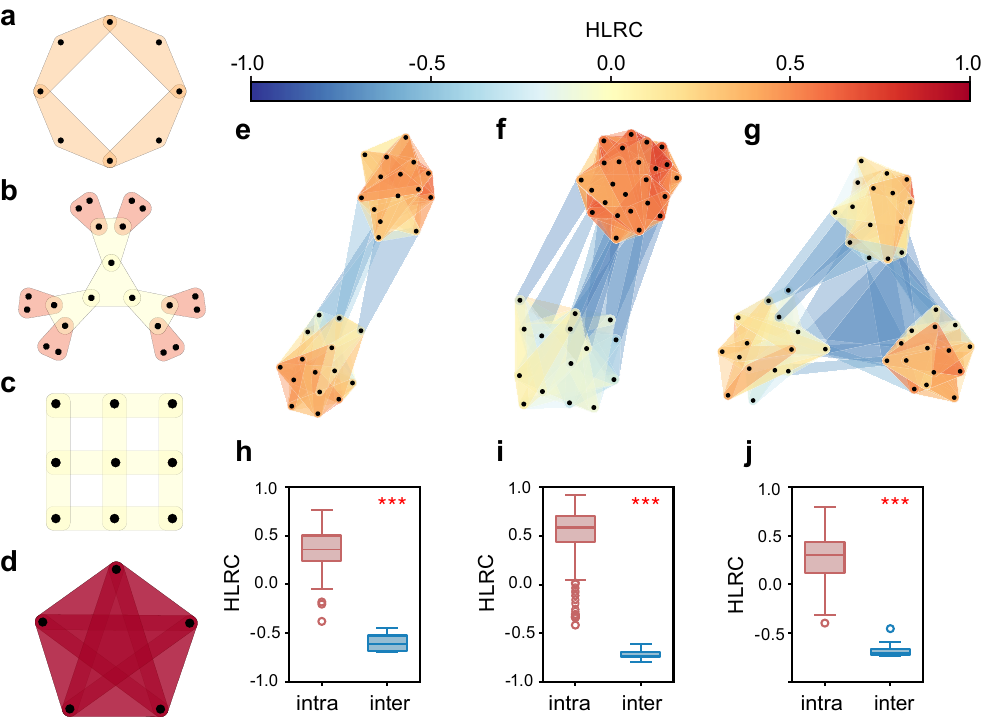}
    \caption{\textbf{Synthetic hypergraphs and their curvature properties.} 
    \textbf{a-d}. Examples of 3-uniform hypergraphs with distinct structural patterns.
    \textbf{a}. A 1-intersecting hypercycle.
    \textbf{b}. A 1-intersecting hypertree. \textbf{c}. A 1-intersecting 2-reguluar hypergrid.
    \textbf{d}. A complete hypergraph. 
    \textbf{e-g}. Synthetic hypergraphs generated using the stochastic block model.
    \textbf{e}. Two equal-sized communities (15 nodes each).
    \textbf{f}. Two unequal-sized communities (15 and 25 nodes).
    \textbf{g}. Three equal-sized communities (15 nodes each). 
    \textbf{h-j}. Distributions of HLRC values for intra- versus inter-community hyperedges corresponding to \textbf{e}-\textbf{g}. Statistical significance was evaluated using the Wilcoxon rank-sum test.}
    \label{fig:1}
\end{figure}

We first applied HLRC to specific classes of 3-uniform hypergraphs, including hypercycles, hypertrees, hypergrids, and complete hypergraphs (see ``Special uniform hypergraphs'' in Methods). The results align with those observed in analogous graph settings. In the case of hypercycles (\textbf{Fig. \ref{fig:1}}a), hyperedges always exhibit non-negative curvature, which reflects the locally cyclic nature of connectivity. This non-negative curvature arises because each node in a hypercycle is connected to a restricted set of neighbors, forming a regular, closed structure that promotes uniformity and cohesion across the hyperedges. For hypertrees (\textbf{Fig. \ref{fig:1}}b), there is a distinct variation in curvature depending on the position of the hyperedge within the tree. Specifically, leaf (terminal) hyperedges tend to have higher curvature compared to the non-terminal hyperedges. This is because they are connected to nodes with fewer neighbors and exhibit reduced local branching, resulting in a higher curvature at these positions. In contrast, non-terminal hyperedges in hypertrees, which are connected to nodes with more neighbors and exhibit more branching, generally have lower curvature. In the hypergrid (\textbf{Fig. \ref{fig:1}}c), the curvature of each hyperedge is zero, capturing the intuition of a locally flat geometry. This suggests that hyperedges in a hypergrid maintain a relatively uniform and non-cyclical structure, with the nodes being uniformly distributed across the grid. Finally, in the complete 3-uniform hypergraph (\textbf{Fig. \ref{fig:1}}d), all hyperedges attain a curvature of one, which is the theoretical upper bound in our definition. This indicates maximal connectivity and saturation, where each node is connected to every other node, leading to the most highly connected structure possible within the given hypergraph model. More special hypergraph visualization under varying parameters are illustrated in \textbf{Supplementary Fig. 1}.

\subsection{Leveraging curvature to capture hypergraph geometry for community detection.}

To assess the ability of HLRC in detecting community structure, we  applied it to synthetic hypergraphs generated by a stochastic block model\cite{ghoshdastidar2014consistency}  (see ``Synthetic hypergraphs'' in Methods). The results, shown in \textbf{Fig. \ref{fig:1}}e–j, comprise three 3-uniform hypergraphs with varying node counts (30, 40, and 45) and number of communities (2, 2, and 3). In each hypergraph, intra-community hyperedges appear with probability 0.1, whereas inter-community hyperedges emerge at probability 0.001. Across all simulations, intra-community hyperedges consistently exhibit significantly more positive curvature than inter-community ones ($p$-value $<$ 0.001; Wilcoxon rank-sum test). This difference stems from their local connectivity patterns: intra-community hyperedges join nodes that share many common neighbors, creating a tightly knit, highly overlapping neighborhood and thus high curvature, whereas inter-community hyperedges bridge nodes with few mutual neighbors, yielding sparse overlap and consequently low curvature. Because HLRC values form a clear bimodal distribution, one can simply threshold or cluster hyperedges by curvature to separate dense, community-internal links from sparse, boundary links.  Examples of synthetic 4-uniform hypergraphs generated using the stochastic block model are shown in \textbf{Supplementary Fig. 2}.

\begin{table}[ht]
    \centering
    \begin{tabular}{l c c c c c c}
        \hline
         & $m$ & $n$ & $d_v$ & $d_e$ & $D_e$ & $\%d_e=2$ \\
        \hline
        Contact High School  & $7818$ & $327$ & $55.6\pm27.1$ & $2.4\pm0.5$ & $5$ & $70.3\%$ \\
        MADStat  & $83331$ & $47311$ & $3.7\pm 7.8$ & $2.1\pm1.1$ & $33$ & $40.8\%$\\
        MAG-10  & $51888$ & $80198$ & $2.3\pm4.6$ & $3.5\pm1.6$ & $25$ & $29.9\%$\\
        \hline
    \end{tabular}
    \caption{\textbf{Summary statistics of individual hypergraph datasets.} For each individual dataset, we report the number of hyperedges ($m$) and nodes ($n$), the mean$\pm$SD of node degree ($d_v$) and hyperedge size ($d_e$), the maximum hyperedge size ($D_e$), and the percentage of pairwise hyperedges ($\%d_e=2$).}
    \label{tab:1a}
\end{table}
\begin{table}[ht]
    \centering
    \begin{tabular}{l c c c c c c }
        \hline
         & $N$ & $m$ & $n$ & $D_e$ & $\%d_e=2$ \\
        \hline
        Stex & $355$ & $18844.2\pm50323.4$ & $433.6 \pm 745.7$ & $7$ & $29.4\%$\\
        Mus & $1944$ & $255.6\pm417.1$ & $24.5\pm 6.6$ & $12$ & $12.3\%$\\
        \hline
    \end{tabular}
    \caption{\textbf{Summary statistics of hypergraph collection datasets.} For each individual dataset, we report the number of hypergraphs ($N$), the mean$\pm$SD of the number of hyperedges ($m$) and nodes ($n$), the maximum hyperedge size ($D$), and the percentage of pairwise hyperedges ($\%d2$).}
    \label{tab:1b}
\end{table}

\begin{figure}
    \centering
    \includegraphics[width=\linewidth]{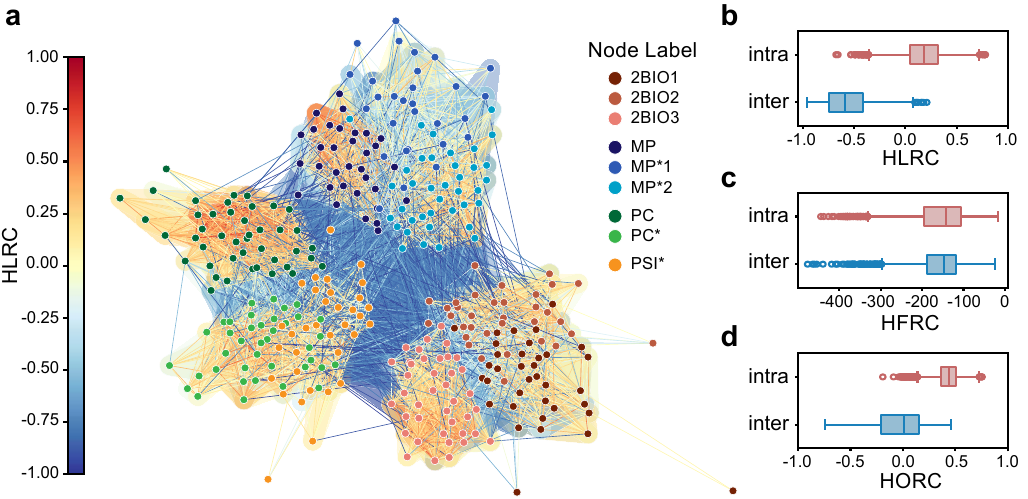}
    \caption{\textbf{Hypergraph curvature analysis of high school contact networks.} 
    \textbf{a}. Contact high school hypergraph representation where nodes correspond to students and hyperedges (colored by HLRC values) represent group interactions. Hyperedges within the same classroom tend to exhibit higher (more positive) curvature than those spanning multiple classrooms. 
    \textbf{b}. Distribution of HLRC values for intra-classroom versus inter-classroom hyperedges. \textbf{c-d}. Corresponding comparisons using alternative curvature measures: HFRC \textbf{c} and HORC \textbf{d}.}
    \label{fig:2}
\end{figure}
We next applied HLRC to the contact high school dataset\cite{Mastrandrea-2015-contact, chodrow2021hypergraph} (\textbf{Table \ref{tab:1a}}), which records proximity interactions among 327 students across nine second-year classes via wearable sensors. Specifically, the classes include three Biology groups (2BIO1, 2BIO2, 2BIO3), three Mathematics–Physics groups (MP, MP*1, MP*2), two Physics–Chemistry groups (PC, PC*), and one Engineering Sciences group (PSI*). The resulting hypergraph comprises 7,818 hyperedges, each representing a group interaction detected within 1-1.5 meters over a 20-second interval. Given that students tend to interact more frequently within their own classrooms, and particularly among classrooms sharing a prefix that indicates a common academic focus, we anticipated that curvature measures would reveal both coarse and fine grained community structures. To test this, we computed HLRC for each hyperedge and also evaluated the two existing hypergraph curvature measures, HORC and HFRC, for comparison. For collision-free visualization of the hypergraph, we applied the Fruchterman-Reingold force-directed algorithm\cite{kobourov2012spring} to its graph projection to obtain two-dimensional node embeddings (\textbf{Fig. \ref{fig:2}}a). Three dominant clusters emerge: MP at the top, 2BIO on the right, and PC and PSI on the left, with tighter subclusters corresponding to individual classrooms. We then colored each hyperedge according to its HLRC value (positive $\rightarrow$ yellow/red; negative $\rightarrow$ light/dark blue) and observed that hyperedges at the boundaries between clusters appear dark blue (strongly negative HLRC), those connecting subclusters within the same cluster are light blue (moderately negative HLRC), and intra-classroom hyperedges range from yellow to red (positive HLRC). This visually striking gradient confirms that HLRC effectively distinguishes boundary hyperedges from those within communities. To quantify this separation and compare different hypergraph curvature measures, we further labeled each hyperedge as “intra-community” if it joins students in the same classroom, and “inter-community” otherwise. As shown in \textbf{Fig. \ref{fig:2}}b-d, both HLRC and HORC sharply distinguish these categories, assigning significantly lower curvature to intra-community hyperedges and higher curvature to inter-community hyperedges, whereas HFRC shows no meaningful separation. Moreover, HLRC achieves this separation with minimal computational cost, whereas HORC requires over a hundred times longer computation time (\textbf{Supplementary Fig. 3}).
These results demonstrates that HLRC delivers accurate higher-order community signals with minimal runtime overhead.

\subsection{Revealing hidden hyperedge patterns across classes through curvature analysis.}

After examining the role of curvature in community detection on hypergraphs with labeled nodes, we next turned to hypergraphs with labeled hyperedges to explore the structural information that HLRC can reveal. We analyzed two co-authorship hypergraphs, Multi‐Attribute dataset (MADStat\cite{ji2022co}) and MAG-10\cite{amburg2020clustering, Sinha-2015-MAG} (\textbf{Table \ref{tab:1a}}). In both datasets, nodes represent authors and hyperedges represent publications linking all co-authors. Each hyperedge in MADStat is annotated with the statistical journal and year of publication, while in MAG-10 it is labeled with the corresponding computer science conference.  By comparing results across domains, i.e., statistics and computer science, we assessed the consistency of HLRC-derived patterns and the extent to which they capture domain-specific structural organization.

\begin{figure}
    \centering
    \includegraphics[width=\linewidth]{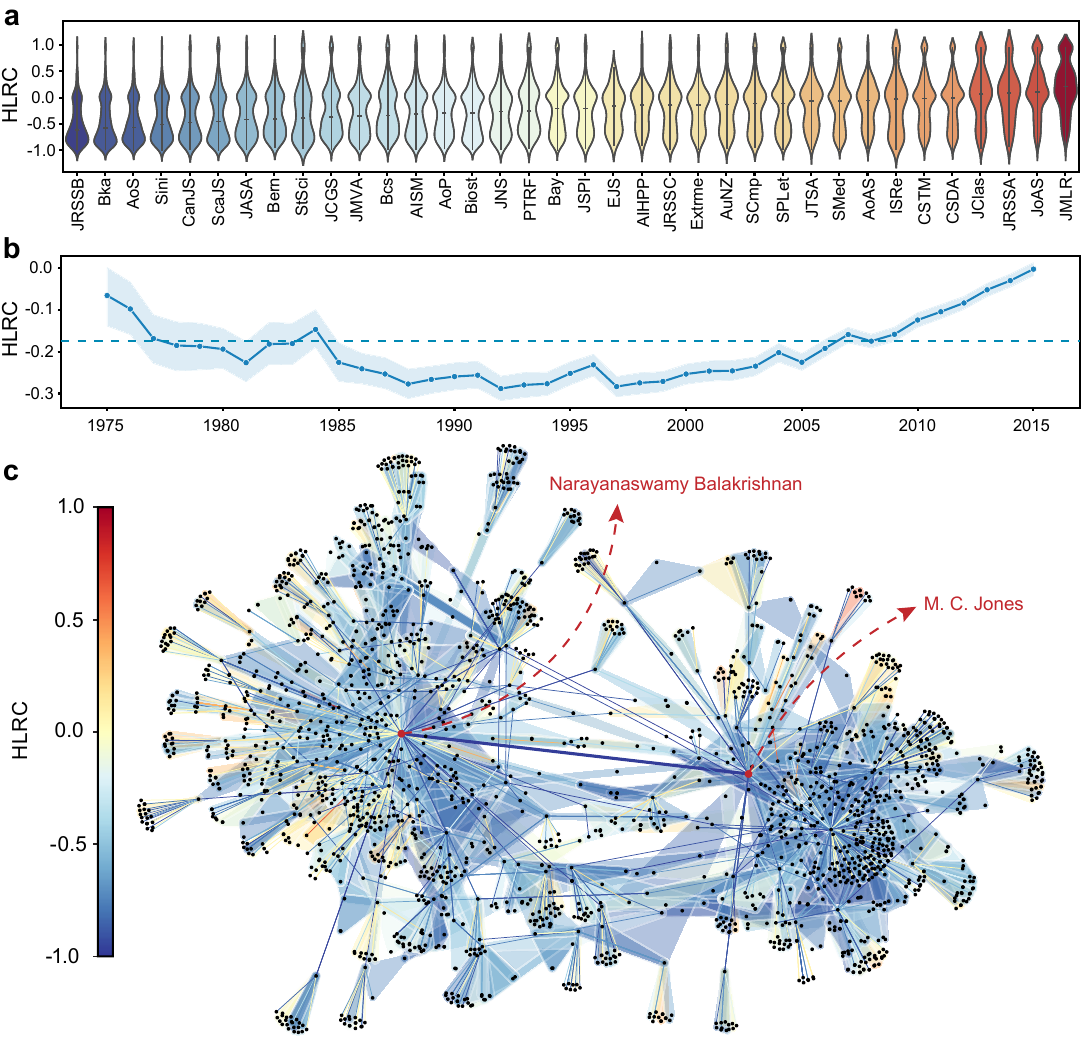}
    \caption{\textbf{Hyperedge curvature distributions and temporal trends in the MADstat co-authorship network.} \textbf{a}. Distribution of hyperedge curvature values across journals reveals a contrast between theoretical journals, which exhibit more negative curvature, and applied journals, which show more positive curvature. \textbf{b}. Temporal trajectory of average hyperedge curvature from 1975 to 2020, characterized by an initial decline (1975-1988), a period of stabilization (1988-2000), and a subsequent increase beginning in 2000. \textbf{c}. Two-hop collaboration subgraph centered on the hyperedge with the most negative curvature, highlighting disconnected author communities with no shared co-authors.}
    \label{fig:3}
\end{figure}

For each co-authorship hypergraph, we computed HLRC for every hyperedge derived a curvature distribution for each publication venue by grouping hyperedges by their associated journal or conference labels. In the MADStat hypergraph, these distributions vary systematically (\textbf{Fig. \ref{fig:3}}a): theoretical journals exhibit consistently lower (more negative) HLRC values, while applied journals tend to show higher curvature. For example, the \textit{Journal of the Royal Statistical Society: Series B} (JRSSB), known for publishing theoretical methodological advances, displays the lowest average HLRC. In contrast, the \textit{Journal of Machine Learning Research} (JMLR), which emphasizes algorithms, shows the highest. This divergence likely reflects distinct collaboration patterns: papers in theoretical journals often involve principal investigators collaborating across broader, less densely connected co-author networks, resulting in hyperedges with sparse overlap, whereas applied research more frequently emerges from tightly knit, cohesive author groups. An analogous trend appears in the MAG-10 dataset (see Supplementary Note 3; \textbf{Supplementary Fig. 4}). Additionally, we examined the temporal dynamics of collaboration by tracking the evolution of mean curvature over time, using each paper’s publication year in the MADStat hypergraph. As shown in \textbf{Fig. \ref{fig:3}}b, the average HLRC declines from 1975 to 1988, remains roughly constant through 2000, and then rises markedly thereafter. The early decline suggests a shift toward less cohesive collaboration patterns, possibly reflecting the diversification of statistical subfields. Conversely, the post-2000 rise coincides with the growth of data-driven and machine learning research, which often involves larger, more interdisciplinary teams with denser co-authorship structure.
Furthermore, we extracted and plotted the two-hop co-authorship subgraph centered on the hyperedge with the lowest HLRC value (\textbf{Fig. \ref{fig:3}}c). In this subgraph, two densely connected but entirely disjoint author communities emerge: one author (N. Balakrishnan) has 186 collaborators and the other (M.C. Jones) 56, but they share no common neighbors. This complete absence of overlap drives the HLRC of that hyperedge to $-0.976$, underscoring the metric’s ability to capture sharp community boundaries and structural gaps in collaboration networks. More one-hop subgraphs centered on the hyperedges with extremely low HLRC value are visualized in \textbf{Supplementary Fig. 4}.

\subsection{Extending curvature-based clustering to hypergraphs collections.}

Finally, we extended our evaluation to two hypergraph collection datasets, Stex\cite{coupette2023ollivier} and Mus\cite{coupette2023ollivier} (\textbf{Table \ref{tab:1b}}), to assess the effectiveness of HLRC in unsupervised hypergraph clustering. In the Stex dataset, each hypergraph represents a StackExchange site, with nodes representing users and hyperedges corresponding to questions annotated with up to five tags. In the Mus dataset, each hypergraph models a musical piece, with nodes as discrete pitch classes and hyperedges as chords sounding for specified durations at particular time offsets. For both collections, we selected three semantically or stylistically coherent groups for clustering analysis: language, religion, and technology-oriented (`nerd') forums from Stex, and Renaissance, Baroque, and later compositions from Mus (The rationale for grouping was explained in Supplementary Note 3). 
\begin{table}[ht]
\centering
\begin{tabular}{l cc cc}
\hline
\multirow{2}{*}{}  & \multicolumn{2}{c}{ARI} & \multicolumn{2}{c}{AMI} \\
\cmidrule(lr){2-5}
 & HLRC & HORC & HLRC & HORC \\
\hline
Stex & $0.536$ & $0.233$ & $0.438$ & $0.169$ \\
Mus  & $0.398$ & $0.205$ & $0.321$ & $0.182$ \\
\hline
\end{tabular}
\caption{\textbf{Clustering performance on hypergraph collections.} Adjusted Rand index (ARI) and adjusted mutual information (AMI) scores are shown for HLRC‐ and HORC‐based embeddings on the Stex and Mus datasets, demonstrating that HLRC yields substantially higher clustering accuracy across both metrics.}
\label{tab:2}
\end{table}

\begin{figure}[ht]
    \centering
    \includegraphics[width=0.8\linewidth]{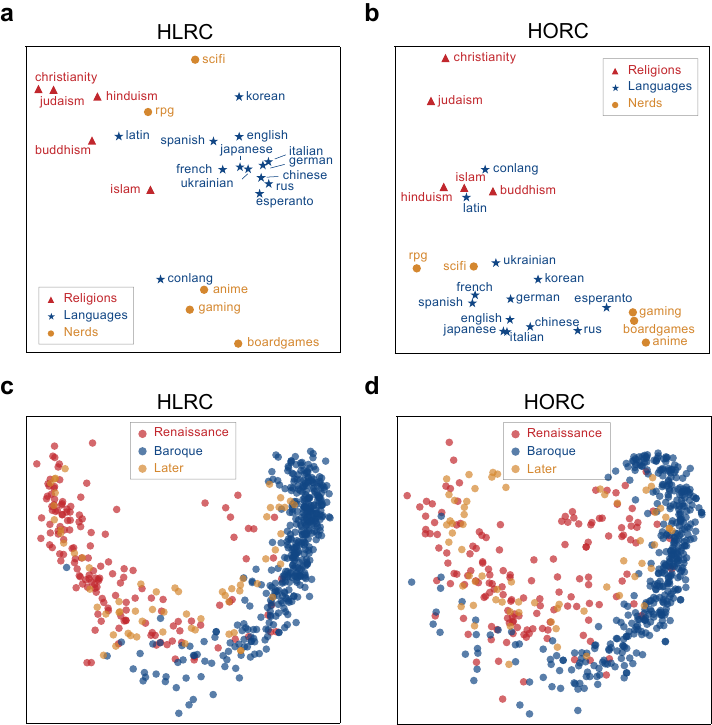}
    \caption{\textbf{kPCA embeddings of curvature histograms reveals semantic and stylistic clusters in hypergraph collections.} \textbf{a-b}. Two-dimensional embeddings of Stex hypergraphs obtained by applying an RBF kernel to curvature histograms followed by kPCA. Features In \textbf{a} are derived from HLRC distributions while Features in \textbf{b} are from HORC distributions. Points are colored by forum category (Religion, Language, Nerds). \textbf{c-d}. Two-dimensional embeddings for Mus hypergraphs, with points colored by musical era (Renaissance, Baroque, Later), based on HLRC (\textbf{c}) and HORC (\textbf{d}). histograms. HLRC-based projections yield more compact and well-separated clusters compared to those derived from HORC.}
    \label{fig:4}
\end{figure}

We then implemented hypergraph clustering to obtain two-dimensional hypergraph embeddings using HLRC and HORC  (see ``Hypergraph clustering'' in Methods). In the Stex dataset, HLRC-based embeddings form sharply defined clusters, with the language community emerging as a distinct and self-contained group, whereas HORC results in more diffuse and overlapping groupings (\textbf{Fig. \ref{fig:4}}a, b). Similarly, in the Mus dataset, embeddings derived from HLRC distinctly separate Baroque compositions from those of the Renaissance and later periods, despite some overlap between the latter two, while HORC-based embeddings show a conflation of Renaissance and Baroque works clustered near the lower center of the plot, with later-period pieces encroaching on the Renaissance region and a noticeably less cohesive Renaissance cluster (\textbf{Fig. \ref{fig:4}}c, d). Quantitatively, HLRC significantly outperforms HORC across both clustering metrics, the adjusted Rand index (ARI\cite{hubert1985comparing}) and the adjusted mutual information score (AMI\cite{vinh2009information}) (\textbf{Table \ref{tab:2}}).  On Stex, the HLRC-based approach achieves an ARI of 0.536 and AMI of 0.438, compared to 0.233 and 0.169, respectively, for HORC. On Mus, HLRC attains an ARI of 0.398 and an AMI of 0.321, markedly surpassing HORC’s scores of 0.205 and 0.182. These results demonstrate that HLRC generalizes effectively across domains and scales, capturing both fine-grained topological variation within individual hypergraphs and global structural patterns across entire collections.

\section{Discussion}

In this work, we introduced  hypergraph lower Ricci curvature (HLRC), a unified curvature metric that reconciles computational efficiency and rich geometric discrimination. Existing hypergraph curvature measures, HORC and HFRC, occupy opposing extremes. HORC provides interpretable and bounded curvature values that reflect meaningful structural and relational properties of hypergraphs. However, this interpretability comes at a steep computational cost due to its complex formulation, and its curvature range is asymmetric, which can complicate comparisons and downstream applications. On the other hand, HFRC achieves near-linear runtime, making it scalable to large datasets, but its curvature values are primarily governed by local degree statistics rather than genuine connection patterns, resulting in limited geometric insight. Moreover, HFRC lacks universal bounds, reducing its interpretability and making it difficult to compare curvature values across different hypergraphs. HLRC bridges this divide by combining HFRC’s local, near-linear scalability (\textbf{Table \ref{tab:3}}; \textbf{Supplementary Fig. 3}) with HORC’s ability to capture nuanced connectivity within a symmetric, bounded interval,  offering a versatile tool for hypergraph analysis.

\begin{table}[ht]
    \centering
    \begin{tabular}{l c c c}
    & \textbf{HLRC}  & \textbf{HFRC} & \textbf{HORC} \\
    \hline
    Theoretical & $O(m\bar{d_e}n)$ & $O\bigl(m\bar{d_e}\bigr)$ & $O(m \bar{d_e}^2 D_v^3)$\\
    \hline
    Contact High School & 40 ms & 6 ms & 3252 ms \\
    MADStat & 311 ms & 93 ms & $>$ 3h \\ 
    MAG-10 & 326 ms & 56 ms  & $>$ 3h \\
    Stex & 192 s& 10 s& 11714 s\\
    Mus & 3 s & 1 s & 146 s\\
     \hline
    \end{tabular}
    \caption{\textbf{Theoretical and empirical runtimes of hypergraph curvature measures.} Top row shows asymptotic time complexities for a hypergraph with $n$ nodes, $m$ edges, average edge size $\bar{d_e}$, and maximum node degree $D_v$. Subsequent rows report wall‐clock runtimes across all threads for each dataset, with HLRC and HFRC implemented in single‐threaded Python and HORC in Julia (single thread for individual datasets, four threads for collections). HORC did not complete within a three-hour timeout on the MADStat and MAG-10 datasets.}
    \label{tab:3}
\end{table}

Evaluated on both synthetic and real-world hypergraphs, HLRC demonstrates robust performance across multiple analytical scales. At the node level, HLRC effectively captures community structure by assigning strongly positive curvature to hyperedges within cohesive groups, while attributing negative curvature to hyperedges that bridge distinct communities. This dual characterization enables clear separation of communities in stochastic block models and human-contact networks, reflecting HLRC’s sensitivity to underlying network organization and its potential for identifying functionally significant connections. At the hyperedge level, HLRC reveals semantic and functional distinctions within co-authorship networks annotated by publication venue. By uncovering structural variations linked to different academic disciplines and exposing temporal trends in collaborative behavior, HLRC provides a powerful geometric lens for probing the dynamics and heterogeneity of scholarly collaboration. Finally, at the global hypergraph scale, embeddings derived from HLRC successfully recover meaningful semantic groupings, effectively summarizing complex higher-order relationships in a compact representation. This global perspective preserves the intrinsic structural identity of the hypergraph, enabling downstream tasks such as clustering, classification, and visualization. Collectively, these results highlight HLRC’s strength as a unifying metric that bridges local, meso-, and global scales in hypergraph analysis.

HLRC opens new avenues for hypergraph analyses across a wide range of domains. In systems biology, hyperedges naturally represent higher-order interactions such as protein complexes or regulatory pathways. Those exhibiting extreme positive curvature often correspond to functionally cohesive complexes or tightly coordinated pathways, while hyperedges with pronounced negative curvature highlight critical interactions that bridge distinct functional modules. This geometric perspective provides a powerful framework for prioritizing experimental investigations and guiding drug target discovery by identifying biologically significant subnetworks. In social and epidemiological contexts, HLRC’s sensitivity to the distinction between intra-community cohesion and inter-community bridging facilitates the detection of tightly knit groups and the identification of pivotal bridge events that govern information flow or disease transmission. Furthermore, HLRC holds promise in advancing hypergraph neural networks by addressing key challenges such as over-squashing and over-smoothing. By selectively amplifying connections associated with negatively curved hyperedges and masking those with highly positive curvature, HLRC helps preserve localized structural features while enhancing global representation learning. Together, these capabilities position HLRC as a versatile and powerful tool for uncovering and leveraging complex higher-order relationships in real-world systems.

Despite these advances, several limitations persist. HLRC is currently formulated for unweighted, undirected hypergraphs and assumes a uniform treatment of neighborhood overlap. Extending HLRC to weighted, directed, and temporal hypergraphs would significantly broaden its applicability and better capture the complexity of real-world interaction patterns. Another important direction involves accounting for heterogeneous roles of nodes within hyperedges, recognizing that contributions within group interactions are often uneven and context-dependent. Furthermore, while HLRC is both interpretable and computationally efficient, its theoretical foundations remain largely heuristic. Establishing formal connections with foundational mathematical frameworks, such as optimal transport theory or curvature flow dynamics adapted to hypergraph settings, could deepen its theoretical rigor and facilitate principled generalizations and more robust analytical tools.

\section{Methods}

\subsection{Hypergraphs.}
Formally, a hypergraph with $n$ nodes and $m$ hyperedges is defined as $\mathcal{H} = \{\mathcal{V}, \mathcal{E}\}$, where $\mathcal{V} = \{v_1,v_2, \dots, v_n\}$ denotes the set of nodes and $\mathcal{E} = \{e_1,e_2, \dots, e_m\}$ denotes the set of hyperedges such that $e_j\subseteq \mathcal{V}$ for $j=1,2,\dots,m$. The structure of a hypergraph $\mathcal{H}$ can be encoded using an incidence matrix $\mathbf{H} \in \{0,1\}^{n \times m}$, where $\mathbf{H}_{ij} = 1$ if node $v_i$ belongs to hyperedge $e_j$, and $\mathbf{H}_{ij} = 0$ otherwise. The degree of a node $v_i$, denoted by $d(v_i)$, is the number of hyperedges that include $v_i$, which can be computed as $d(v_i) = \sum_{j=1}^m \mathbf{H}_{ij}$. Similarly, the degree of a hyperedge $e_j$, denoted by $d(e_j)$, is the number of nodes it contains, which can be computed as $d(e_j) = \sum_{i=1}^n \mathbf{H}_{ij}$. Two nodes are considered adjacent if they belong to the same hyperedge. The neighborhood of a node $v_i$, denoted by $\mathcal{N}(v_i)$, comprises all nodes adjacent to $v_i$, with its size  $n(v_i) = |\mathcal{N}(v_i)|$. Likewise, the neighborhood of a hyperedge $e_j$, denoted by $\mathcal{N}(e_j)$, represents the intersection of the neighborhoods of all nodes in $e_j$, with its size $n(e_j) = |\mathcal{N}(e_j)|$. For simplicity, we used $v$ and $e$ to denote nodes and hyperedges, $d_v$ and $d_e$ for their respective degrees,  and $n_v$ and $n_e$ for their neighborhood sizes in defining HLRC and special hypergraphs.

\subsection{Special uniform hypergraphs.}
To derive closed-form expressions for HLRC in uniform hypergraphs, we first introduce a set of standard properties. A hypergraph is said to be connected if there exists a path between any pair of nodes through a sequence of hyperedges. If every hyperedge connects exactly $k$ nodes, the hypergraph is called $k$-uniform, i.e., $d_e=k$ for $\forall e \in \mathcal{E}$. A hypergraph is $r$-regular if every node lies in exactly $r$ hyperedges. We say it is $s$-intersecting when any two distinct hyperedges meet in exactly $s$ nodes, and $c$-coocurrent when every adjacent node pair appears together in precisely $c$ hyperedges, i.e. $v_i \sim v_j$, then $|\{e\in \mathcal{E}: \{i,j\}\subseteq e\}|=c$. Notably, a simple graph emerges as a special case when $k = 2$ and $s = c = 1$.

In a complete $k$-uniform hypergraph, every possible $k$-subset of nodes forms a hyperedge. Consequently, the HLRC of any hyperedge in such a hypergraph equals 1, attaining the theoretical upper bound in the densest possible configuration. We next consider a hypercyle — a cyclic sequence of hyperedges $e_1, e_2, \ldots, e_m$ where each consecutive pair shares at least one node ($e_j \cap e_{j+1} \neq \emptyset$ for $j = 1, \ldots, m - 1$), and the final edge $e_m$ intersects $e_1$. For a $k$-uniform, $s$-intersecting hypercycle with $m$ hyperedges, the HLRC admits closed-form expressions that vary with the relationship between $k$ and $s$. For example, when the hypercyle is large enough i.e.$m\ge2k$, we have
\[
    \text{HLRC}(e)=\begin{cases}
    \frac{k/2-2s}{k-1}+\frac{k/2+2s-1}{2k-s-1}, &k>2s \\
    \frac{s-1}{3s-1} , &k=2s \\
    \vdots &\text{(intermediate regimes)} \\
    0, &k=s+1
    \end{cases}.
\]
These formulas capture a smooth transition in curvature as a function of intersection size $s$. As $s$ increases, HLRC generally decreases, reaching zero when $k = s + 1$, reflecting a setting where each node has many neighbors but shares none within a single hyperedge. (The formulas for the $m<2k$ case are detailed in Supplementary Note 2.)  A third canonical class is the hypertree, which imposes acyclicity by requiring that each hyperedge corresponds to the node set of a connected subtree within an underlying tree structure\cite{brandstadt1998dually}. For a $k$-uniform, $r$-regular, 1-intersecting hypertree, the HLRC evaluates to $\text{HLRC}(e) = \frac{2}{r} - 1$ for non-terminal hyperedges, and $\text{HLRC}(e) = \frac{(r + 1)k - 2r}{2r(k - 1)}$ for terminal ones. Finally, hypergrids arise from sliding an $r$-node window along simple paths of length $r$ in a fixed lattice graph. Formally, a hypergraph is defined as a hypergrid\cite{coupette2023ollivier} if there exists a lattice $\mathcal{L} = \{\mathcal{V}, \mathcal{E}_\mathcal{L}\}$ such that $\mathcal{E} = \{e \in \binom{\mathcal{V}}{r} \mid e \text{ corresponds to a path of length } r \text{ in } \mathcal{L} \}$. When this construction yields a $k$-uniform, 1-intersecting, 2-regular hypergrid, the HLRC of each hyperedge is zero. Complete derivations for these special cases are provided in Supplementary Note 2.

\subsection{Synthetic hypergraphs.}

A hypergraph stochastic block model (HSBM\cite{ghoshdastidar2014consistency}) generalizes the classic graph-based stochastic block model by allowing hyperedges to connect more than two nodes and by using community labels to govern the probability of each multi-way interaction. In an $k$-uniform HSBM, the set of $n$ nodes is first partitioned into $r$ blocks or “communities”. For every subset $\mathcal{S} \subseteq \mathcal{V}$ of size $k$, the hyperedge $\mathcal{S}$ is added with probability $P_{i_1i_2\dots i_k}$ where $i_j$ is the community of $j^{th}$ node in $\mathcal{S}$.  In this study, it is simplied to two parameters by setting $P_{i_1i_2\dots i_k}=a$ if all $i_j$ are equal and $P_{i_1i_2\dots i_k}=b$ otherwise, capturing assortative structure via within-block probability $a$ and cross-block $b$. This higher-order model avoids information loss from reducing to pairwise links and underpins sharp theoretical results on weak and exact recovery thresholds, as well as efficient spectral and semidefinite algorithms for community detection.

\subsection{Hypergraph clustering.}

To assess whether hyperedge curvature effectively captures group‐level organization within hypergraphs, we implemented the following four‐stage pipeline. First, for each hypergraph in our collection, we calculated the curvature value of every hyperedge using both HLRC and HORC measures. Curvature scores were computed on the full hypergraph topology without subsampling, ensuring that edges of all sizes and intersection patterns contributed to the analysis. Secondly, the resulting curvature values were aggregated into frequency histograms. We partitioned the curvature range into bins of width 0.05, resulting in 40 bins spanning the HLRC interval $(-1, 1]$ and 60 bins spanning the HORC interval $[-2, 1]$. Each histogram was normalized to sum to one, producing a probability‐density representation for each hypergraph’s curvature distribution. Thirdly, for nonlinear dimensionality reduction of the curvature‐histogram features, we first computed a pairwise similarity matrix $\textbf{D}$ using the radial basis function (RBF) kernel\cite{buhmann2000radial} on the $N\times B$ histogram matrix $\textbf{G}$, where $N$ is the number of hypergraphs in the dataset and $B$ is the number of bin. Specifically, \[\textbf{D}_{ij} = \exp{(-\gamma ||\textbf{G}_i-\textbf{G}_j||^2)}\] with bandwidth $\gamma$ being $1/B$. We then performed kernel principal component analysis (kPCA\cite{scholkopf1997kernel}) directly on this precomputed kernel matrixm requesting two components, and setting a convergence tolerance of $10^{-5}$ and a maximum of 2000 iterations. Lastly, the two-dimensional kPCA embeddings were clustered using k-means\cite{macqueen1967some}, with the number of semantic groups in each dataset (e.g., three groups for language, religion, and “nerd” communities in the Stex collection). Cluster assignments were compared to ground‐truth group labels using two complementary metrics: the adjusted Rand index (ARI\cite{hubert1985comparing}) and the adjusted mutual information score (AMI\cite{vinh2009information}). Both scores correct for chance agreement, with values ranging from 0 (no better than random) to 1 (perfect recovery).

\section*{Code and Data Availability}
The complete source code and associated datasets supporting this study are publicly accessible at \url{https://github.com/shiyi-oo/hypergraph-lower-ricci-curvature.git}

\bibliographystyle{unsrt}

\newpage
\begingroup
\let\oldthesection\thesection
\newcounter{oldsec}%
\setcounter{oldsec}{\value{section}}%
\let\oldfigurename\figurename

\renewcommand{\thesection}{S\arabic{section}}%
\setcounter{section}{0} 
\renewcommand{\figurename}{Supplementary Fig.}
\setcounter{figure}{0}
\begin{center}
  \textbf{\Large Supplement of ``Lower Ricci Curvature for Hypergraphs"}
\end{center}
\section{Recap of prior hypergraph curvature definitions.}
In this note, we briefly reviewed the definitions of two existing generalizations of graph curvature to hypergraphs: Hypergraph Olliveir Ricci Curvature (HORC\cite{coupette2023ollivier}) and Hypergraph Forman Ricci Curvature (HFRC\cite{leal2021forman}).

\subsection{Definition of HORC.}
Ollivier–Ricci curvature on a hyperedge is built by comparing how probability mass is distributed among its nodes. Concretely, for an unweighted, undirected hypergraph $\mathcal{H}=\{\mathcal{V},\mathcal{E}\}$, pick a hyperedge $e$ of degree $d_e$. The HORC is defined by
\begin{equation}
    \kappa(e) :=1 - \mathrm{AGG}(e),
\end{equation}
where the aggregation operator $AGG$ averages pairwise transport costs:
$$\mathrm{AGG}(e) :=  \frac{2}{d_e\,(d_e-1)} \sum_{\{v_i,v_{j}\}\subseteq e} W_{1}\bigl(\mu_{i},\mu_{j}\bigr).$$
Here, $d_e$ is the size of $e$ and $W_{1}(\cdot,\cdot)$ denotes the 1-Wasserstein distance between the probability measures $\mu_i$ and $\mu_j$ supported on nodes $v_i$ and $v_j$, respectively. In our implementation, each node $v$ carries a simple random-walk distribution over its neighbors (assigning mass $1/n_v$ to each neighbor and zero probability of remaining at $v$). Alternative aggregation rules and measure definitions are well documented in \cite{coupette2023ollivier}; we selected this particular pairing of $\mathrm{AGG}$ and $\mu$ so that the HORC formulation mirrors the HLRC philosophy for direct comparison. Regardless of which reasonable choices one makes for $\mathrm{AGG}$ and $\mu$, the resulting HORC value always lies in the interval $[-2, -1]$.

\subsection{Definition of HFRC.}
HFRC is a purely combinatorial measure that assigns each hyperedge a curvature based on two simple ingredients: the size of the hyperedge (how many nodes it contains) and how many other hyperedges each of those nodes participates in. Formally, for an unweighted, undirected hypergraph, HFRC for a hyperedge $e$ is given by
\begin{equation}
    F(e) = 2d_e-\sum_{v\in e}d_v,
\end{equation}
where $d_e$ is the size of $e$ and each $d_v$ is the number of hyperedges incident to node $v$.
Furthermore, HFRC values lie in a range determined by extreme configurations of node degrees within $e$. In particular, if every node $v\in e$ has exactly $m$ incident hyperedges (so $\sum_{v\in e}d_v=md_e$), then $F(e)=d_e(2-m)$, which attains its minimum when $m$ is as large as possible. Conversely. $F(e)$ achieves its maximum value of $d_e$ if each node in $e$ participates only in $e$ itself (i.e. $\sum_{v\in e}d_v=d_e$). Thus, unlike HORC, which always lies in $[-2,1]$, HFRC can vary from $d_e(2-m)$ up to $d_e$.

\section{Bounds, computation complexity and special hypergraphs.}

In this note, we would firstly prove the lower and upper bound of HLRC, then compute its computational complexity and lastly detail the proofs of HLRC on special hypergraphs.

\subsection{Bounds of HLRC.} 
\begin{theorem}
Let $\mathcal{H}=\{\mathcal{V},\mathcal{E}\}$ be an unweighted, undirected simple hypergraph, and let $e\in\mathcal{E}$ be any hyperedge of size $d_e>1$. Then its HLRC satisfies $$-1<\mathrm{HLRC}(e)\le1.$$
\end{theorem}
\begin{proof}
Since each $v\in e$ is adjacent to the other $d_e-1$ vertices in $e$ plus at least $n_e$ additional shared nodes by all $v\in e$, we have $n_v \geq d_e - 1 + n_e$. Then the first term of HLRC, $\sum_{v\in e}\frac{1}{n_e}$ is bounded above by $\frac{d_e}{d_e-1+n_e}$.
Also we have both $\frac{n_e + d_e/2 - 1}{\max_{v\in e} n_v}$ and $
\frac{n_e + d_e/2 - 1}{\min_{v\in e} n_v}$ bounded above by $ \frac{n_e + d_e/2 - 1}{d_e-1+n_e}$, by simply replacing $n_v$ by $d_e - 1 + n_e$.
One obtains
$$
\mathrm{HLRC}(e) \le
\frac{d_e}{d_e-1+n_e} -1+2\cdot\frac{n_e + d_e/2 - 1}{d_e-1+n_e}=1.
$$
For the lower bound, each of the three summands in the definition of $\mathrm{HLRC}(e)$ is strictly positive, so
$$
\mathrm{HLRC}(e)>0 -1 + 0 + 0=-1.
$$
\end{proof}

\subsection{Computational Complexity of HLRC.}
\begin{remark}
    Computational complexity of HLRC is $O(m\bar{d_v}n)$ for a hypergraph of $n$ nodes, $m$ edges, and average edge‐size $\bar {d_v}$.
\end{remark}

\begin{proof}
The computation involves three key steps. First, we precompute node neighborhoods $\mathcal{N}(v)$ by iterating through all $m$ hyperedges and their nodes, recording adjacencies. This requires $O(m\bar{d_v}^2)$ operations where $\bar{d_v}$ is the average hyperedge size, as each node pair within a hyperedge must be processed. Second, we compute edge neighborhoods $\mathcal{N}(e)$ via set intersections of $\mathcal{N}(v)$ for all $v \in e$, costing $O(m\bar{d_v}n)$ due to $O(n)$-time intersections across $\bar{d_v}$ nodes per hyperedge. Finally, calculating HLRC for all hyperedges requires $O(m\bar{d_v})$ operations for aggregating terms. The total complexity combines these steps:
$O(m\bar{d_v}^2) + O(m\bar{d_v}n) + O(m\bar{d_v}) = O(m\bar{d_v}n)$.
\end{proof}

\subsection{HLRC on special uniform hypergraphs.}
\begin{theorem}
    For any hyperedge $e$ in a $k$-uniform complete hypergraph, $\mathrm{HLRC}(e)=1$.
\end{theorem}

\begin{proof}
In the complete hypergraph on $n$ nodes, each node $v$ is adjacent to all other $n-1$ nodes s.t. $\mathcal{N}(v)=\mathcal{V}\backslash\{v\}$ and $n_v=n-1$. For any fixed hyperedge $e$ of size $k$, the set of common neighbors of its vertices is percisely $\mathcal{V}\backslash e$ which has size $n-k$. By definition, $$\mathrm{HLRC}(e) = \frac{k}{n-1} - 1 + \frac{n-k + k/2 - 1}{n-1} + \frac{n-k + k/2 - 1}{n-1} = 1. $$
\end{proof}
\begin{theorem}
     For any hyperedge $e$ in a $k$-uniform, $s$-intersecting hypercycle with $m\ge 2$ hyperedges,
\begin{align*}
    \mathrm{HLRC}(e)=\begin{cases}
    \frac{k/2-2s}{k-1}+\frac{k/2+2s-1}{2k-2s-1}, & k>2s, m = 2\\
    \frac{k/2-2s}{k-1}+\frac{k/2+2s-1}{2k-s-1}, & k>2s, m\geq 3\\
    1 , &k=2s,m = 3 \\
    \frac{s-1}{3s-1} , &k=2s,m \ge 4 \\
    \dots, &\text{(intermediate regimes)} \\
    1, &k=s+1, m < 2k \\
    0, & k=s+1, m \geq 2k
    \end{cases}.
\end{align*}
\end{theorem}

\begin{proof}
Under the setting of $k$-uniform, every hyperedge has exactly $k$ nodes s.t. $d_e=k$. We therefore focus on how many neighbors each nodes in a given hyperedge $e$ has, and the number of shared neighbors, under different overlaps $s$.

First, consider the case $k>2s$. Here $e_j$ shares exactly $s$ nodes with the preceding hyperedge and another $s$ nodes with the following hyperedge, leaving $k-2s$ nodes unique to $e_j$. So we can rewrite hyperedge as $e_j = \{A_s, B_{k-2s}, C_s\}$, where $A_s = \{v_1,...,v_s\}$, $B_{k-2s} = \{v_{s+1},...,v_{k-s}\}$, $C_s = \{v_{k-s+1},...,v_k\}$.  $A_s$ is the overlap set between $e_j$ and $e_{j-1}$, and $C_s$ is the overlap set set between $e_j$ and $e_{j+1}$. 
When $m\ge 3$, any node in $A_s$ or $B_s$ has $k-1$ neighbors within $e_j$ plus an additional $k-s$ neighbors in the overlapping hyperedge, for a total of $2k-s-1$, i.e. $n(v_i) = 2k-s-1,\forall v_i\in A_s\cup B_s$. In contrast, each nodes in $B_{k-2s}$ sees only the other $(k-1)$ vertices of $e_j$, so $n(v_i) = k-1,\forall v_i\in B_{k-2s}$. Trivially, this leads to $n(e)=0$. Substituting these counts into the curvature definition gives 
\begin{align*}
    \mathrm{HLRC}(e) &= \sum_{v_i\in A_s \cup C_s}\frac{1}{2k-s-1} + \sum_{v_i\in B_{k-2s}}\frac{1}{k-1} + \frac{k/2-1}{2k-s-1} + \frac{k/2-1}{k-1}-1\\
    &= \frac{k/2 + 2s-1}{sk-s-1} + \frac{k/2-2s}{k-1}. \quad\quad\quad (k>2s, m\ge 3)
\end{align*}
When $m=2$, $e_{j-1}$ and $e_{j+1}$ are the same hyperedge, thus any node in $A_s$ or $B_s$ has $k-1$ neighbors within $e_j$ plus an additional $k-2s$ neighbors in the overlapping hyperedge, rather than $k-s$ in the prior cases. The other quantities keep the same, thus 
\begin{align*}
    \mathrm{HLRC}(e) &= \sum_{v_i\in A_s \cup C_s}\frac{1}{2k-2s-1} + \sum_{v_i\in B_{k-2s}}\frac{1}{k-1} + \frac{k/2-1}{2k-2s-1} + \frac{k/2-1}{k-1}-1\\
    &= \frac{k/2 + 2s-1}{sk-2s-1} + \frac{k/2-2s}{k-1}. \quad\quad\quad (k>2s, m = 2)
\end{align*}

Next, when $k=2s$, the hyperedge splits evenly into two overlap regions of size $s$, i.e $e_j = \{A_s, B_s\}$, where $A_s = \{v_1,...,v_s\}$ and $B_s = \{v_{s+1},...,v_{2s}\}$. When $m\ge4$, every node then has $k-1=2s-1$ neighbors inside $e_j$ and $s$ neighbors in the adjacent hyperedge, totaling $3s-1$. It follows that
\begin{align*}
\mathrm{HLRC}(e) &= \sum_{v_i\in e_j} \frac{1}{3s-1} + \frac{k/2-1}{3s-1}+\frac{k/2-1}{3s-1}-1\\ 
&= \frac{s-1}{3s-1}.
\quad\quad\quad (k=2s, m \ge 4)
\end{align*}
When $m=3$, every node still has $k-1=2s-1$ neighbors inside $e_j$ and $s$ neighbors in the adjacent hyperedge, totaling $3s-1$. But the $s$ neighbors are their common neighbors under this setting, s.t. $n(e)=s$. It follows that
\begin{align*}
\mathrm{HLRC}(e) &= \sum_{v_i\in e_j} \frac{1}{3s-1} + \frac{k/2+s-1}{3s-1}+\frac{k/2+s-1}{3s-1}-1\\ 
&= 1.
\quad\quad\quad (k=2s, m =3 )
\end{align*}
The derivations for the intermediate regimes—which interpolate as the relative values of $k$, $s$, and $m$ vary—follow exactly the same reasoning as the cases shown above and are therefore omitted.

When $k=s+1$ and $m\ge2k$, for hyperedge $e_j$ contains $\{v_1,v_2,...,v_s,v_{s+1}\}$, we have $\{v_1,v_2,...,v_s\}$ exists in $e_{j-1}$ and $\{v_2,...,v_s,v_{s+1}\}$ exists in $e_{j+1}$. It is easy to prove that each $v_i$ belongs to exactly $s+1$ consecutive hyperedges: $v_1$ exists in $\{e_{j-s},...,e_j\}$, ${v_2}$ exists in $\{e_{j-2},...,e_{j+1}\}$ and in general $v_i$ appears in in $\{e_{j-s+i-1},...,e_{j+i-1}\}$. Within $e_j$ each node has $k-1=s$ neighbors, and it gains exactly one new neighbor from each of the other $s$ hyperedges it belongs to, for a total neighborhood size $n(v_i) = 2s$. Substituting $n(v_i)=2s$ into HLRC definition gives
\begin{align*}
    \mathrm{HLRC}(e) & = \sum_{v_i\in e_j}\frac{1}{2s} + \frac{k/2-1}{2s}+\frac{k/2-1}{2s}-1 \\
    &= 0. \quad\quad\quad (k=s+1, m \ge 2k)
\end{align*}

Finally, When $k=s+1$ and $m < 2k$ hyperedges, every node in $e_j$ becomes adjacent to all $n-1$ other vertices in the hypercycle, and $n(e_j) = n-k$. A direct check then shows $\mathrm{HLRC}(e)=1$, recovering the complete hypergraph case.
\end{proof}

\begin{theorem}
    For any non-terminal hyperedge $e$ in a $k$-uniform, $r$-regular, 1-intersecting hypertree, $\mathrm{HLRC}(e) = \frac{2}{r}-1$. For any terminal hyperedges $e$, $\mathrm{HLRC}(e) = \frac{(r+1)k-2r}{2r(k-1)}$.
\end{theorem}

\begin{proof}
Under the setting of $k$-uniform $r$-reguluar, each node of $e$ belongs to $r$ hyperedges each of size $k$. By the acyclic property of tree,those neighbor-sets never overlap. Hence every node in $e$ has exactly $r(k-1)$ neighbors with no shared common neighbors $(n_e=0)$. Plugging these values into HLRC definition gives
$$
\mathrm{HLRC}(e) = \sum_{v_i\in e_j}\frac{1}{r(k-1)} + \frac{k/2-1}{r(k-1)} + \frac{k/2-1}{r(k-1)} - 1 = \frac{2}{r}-1,
$$
which is negative for any $r\geq 2$.

By contrast, a terminal (leaf) hyperedge intersects the rest of the tree in exactly one node. That single overlap node again has $r(k-1)$ neighbors, while each of the other $k-1$ vertices sees only the $k-1$ neighbors within the terminal hyperedge itself. Summing their contributions and subtracting one yields
$$
\mathrm{HLRC}(e) = \frac{1}{r(k-1)} + \frac{k-1}{k-1} + \frac{k/2-1}{r(k-1)} + \frac{k/2-1}{k-1}-1 = \frac{(r+1)k-2r}{2r(k-1)}.
$$
\end{proof}

\begin{theorem}
    For any hyperedge $e$ in a $k$-uniform, 1-intersecting, 2-regular hypergrid, $\mathrm{HLRC}(e)= 0$.
\end{theorem}

\begin{proof}
In a $k$-uniform, 1-intersecting, 2-regular hypergrid, every hyperedge has exactly $k$ vertices, each node belongs to two hyperedges, and any two hyperedges meet in exactly one node. Consequently, each node of a given hyperedge $e$ has $2(k-1)$ neighbors (the other $k-1$ vertices in $e$ plus $k-1$ node from the adjacent hyperedge). And there is no common neighbors given the structure of hypergrid, thus $n_e=0$. Plugging these values into the HLRC formula,
$$
\mathrm{HLRC}(e) = \sum_{v_i\in e_j}\frac{1}{2(k-1)} + \frac{k/2-1}{2(k-1)} + \frac{k/2-1}{2(k-1)} - 1 = 0
$$
\end{proof}

\section{Additional details in synthetic and real data analysis.}

In this note, we will provide more special hypergraph visualization under varying $k, r, s$, and more hypergraph stochastic block model (HSBM\cite{ghoshdastidar2014consistency}) generated hypergraphs. In addition, We also include a comprehensive datasets description, expanded results, and a concise review of the clustering objective metrics.

\subsection{Special uniform hypergraphs.}
\textbf{Supplementary Fig. \ref{SIfig:1}} illustrates HLRC on various uniform hypergraphs as their parameters vary. All complete $k$–uniform hypergraphs attain the maximum curvature of 1 (\textbf{Supplementary Fig. \ref{SIfig:1}a-c}). In the hypercycle panels (\textbf{Supplementary Fig. \ref{SIfig:1}d-f}), we show three regimes of $(k,s,m)$: when $k=4,s=2$ (so $k=2s$), HLRC is positive at 0.2; when $k=4,s=3,m=8$ (so $k=s+1$ and $m\ge2k$), HLRC falls to 0; and when $m$ decreases to 7 (so $k=s+1,m\le2k-1$), HLRC saturates at 1. For hypertrees, extending the tree depth from 3 to 4 leaves both terminal and non-terminal hyperedge curvatures unchanged. However, increasing the regularity $r$ from 2 to 3 (with fixed $k$ and overlap) drives HLRC of non-terminal hyperedges from 0 down to –0.33 and that of terminal hyperedges from 0.625 to 0.5, reflecting the increased neighborhood size without changing edge size or shared-neighbor structure. Finally, when hyperedge size grows from 3 to 4 in a 3-regular, 1-intersecting tree, non-terminal hyperedge curvature remains determined solely by $r$, while terminal hyperedge curvature rises slightly from 0.50 to 0.56, owing to the larger clique-like substructure at the leaves.

\begin{figure}[ht]
    \centering
    \includegraphics[width=0.8\linewidth]{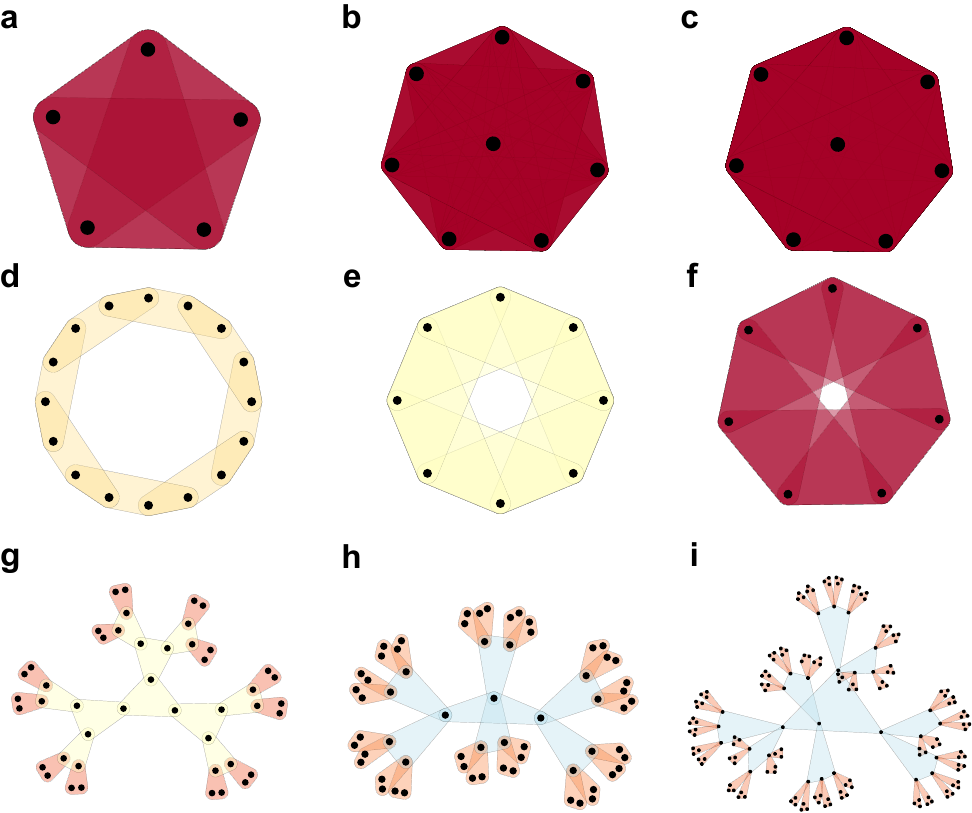}
    \caption{\textbf{HLRC on special uniform hypergraphs under varying parameters.}
    \textbf{a–c}. Complete $k$-uniform hypergraphs on $n$ nodes with $(k,n)=(4,5),(3,8),(5,8)$, each attaining the maximal HLRC value of 1.
    \textbf{d–f}. $k$–uniform, $s$-intersecting hypercycles with $m$ hyperedges in three regimes, $(k,s,m)=(4,2,8),(4,3,8),(4,3,7)$. \textbf{d}. This is the case when $k=2s, m\ge2k$, each hyperedge has HLRC being 0.2. \textbf{e}. This is the case when $k=s+1, m\ge2k$, each hyperedge has HLRC being 0. \textbf{f}. This is the case when $k=s+1, m<2k$, each hyperedge has HLRC being 1.
    \textbf{g-i}. $k$–uniform, $r$–regular, 1–intersecting hypertrees with $(k,r,\text{tree depth})= (3,2,4),(3,3,3),(4,3,3)$.
    \textbf{g}. Compared with \textbf{Fig. 3}b in the main text, increasing depth from 3 to 4 leaves both terminal and non-terminal HLRC unchanged.
    \textbf{h}. Raising regularity $r$ from 2 to 3 (with fixed $k$ and overlap) drives non-terminal hyperedge HLRC from 0 to –0.33 and terminal hyperedge HLRC from 0.625 to 0.50.
    \textbf{i}. Enlarging hyperedge size from 3 to 4, non-terminal hyperedge HLRC hold constant while terminal hyperedge HLRC rises from 0.50 to 0.56.}
    \label{SIfig:1}
\end{figure}

\subsection{Synthetic HSBM generated hypergraphs.}
To test HLRC’s ability to recover community structure, we generated a series of synthetic $k$-uniform hypergraphs via a stochastic block model and visualized both their topology and curvature distributions in \textbf{Supplementary Fig.\ref{SIfig:2}}. In the two-community settings (\textbf{Supplementary Fig.\ref{SIfig:2}}a-b), we fix $k=4$ and compare equal (15+15 nodes) versus imbalanced (20+40 nodes) block sizes, sampling intra-community edges at 0.1 and inter-community edges at 0.001. In the three-community cases (\textbf{Supplementary Fig.\ref{SIfig:2}}c-d), we sample at 0.01 and 0.0001 to avoid overly dense graphs, with community sizes of 15+15+15 and 40+30+20 nodes respectively. Each hypergraph layout is colored by ground-truth block and edge shading indicates HLRC value. In every scenario, HLRC cleanly highlights intra-community hyperedges with positive curvature and inter-community bridges with negative curvature. The lower row (\textbf{Supplementary Fig.\ref{SIfig:2}}e-h) quantifies these observations: boxplots of HLRC for intra- versus inter-community edges show a clear bimodal separation and highly significant differences (Wilcoxon rank-sum, $p$-value$<$0.001). These results confirm that HLRC robustly discriminates community structure across balanced, unbalanced, and multi-block hypergraphs.

\begin{figure}[ht]
    \centering
    \includegraphics[width=\linewidth]{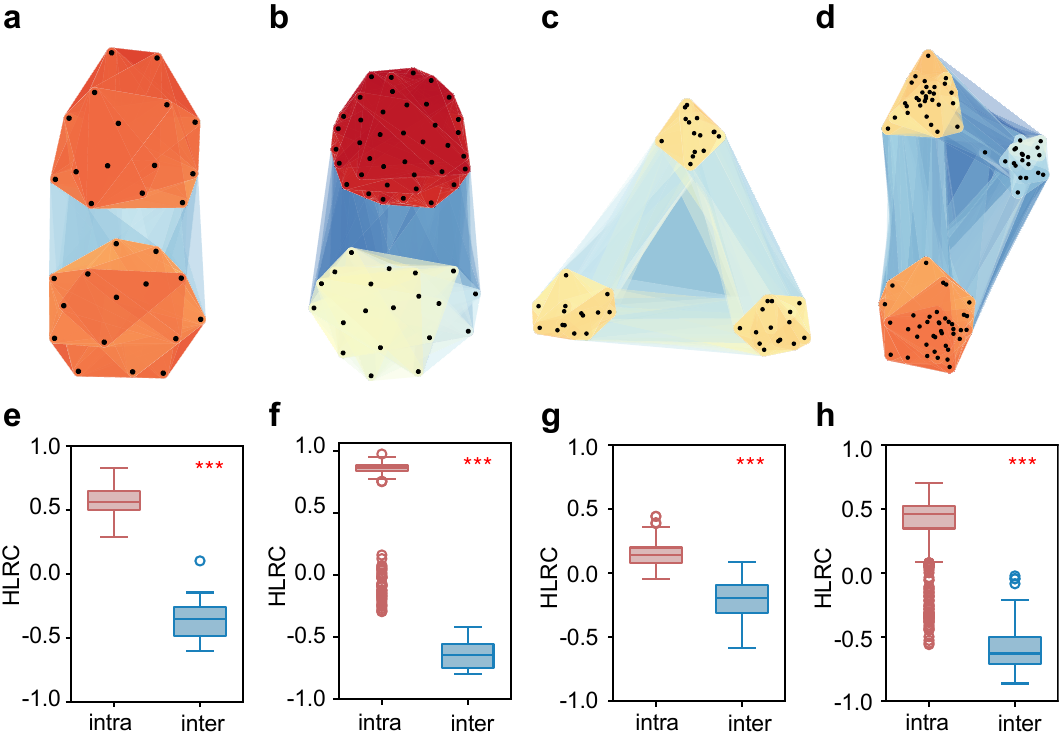}
    \caption{\textbf{4-uniform HSBM-generated hypergraphs with balanced and unbalanced communities.} 
    \textbf{a-b}. Two-community cases where intra-community hyperedges are drawn with probability 0.1 and inter-community hyperedges with probability 0.001. 
    \textbf{a}. Two equal-size communities of 15 nodes each. 
    \textbf{b}. Unequal community sizes of 20 and 40. 
    \textbf{c-d}. To avoid overly dense hypergraphs in the three-community cases, sampling probabilities are reduced to 0.01 (intra) and 0.0001 (inter). 
    \textbf{c}. Three equal‐size communities of 15 nodes. 
    \textbf{d}. Three unequal-size communities of 40, 30, and 20 nodes. 
    \textbf{e–h}. HLRC distributions for intra- versus inter-community hyperedges in \textbf{a}-\textbf{d} with significance assessed via Wilcoxon rank-sum tests.}
    \label{SIfig:2}
\end{figure}

\subsection{Runtime comparison.}

In this runtime comparison experiment, all synthetic hypergraphs were generated according to the Chung–Lu hypergraph model\cite{kaminski2019clustering} by first prescribing each node’s target degree $d(v_i)$ and each hyperedge’s size $d(e_j)$, then sampling incidences so that the total node‐degree volume equals the total hyperedge‐size volume:  
$\mathrm{vol}(V)=\sum_{i=1}^n d(v_i)=\sum_{j=1}^m d(e_j).$
\textbf{Supplementary Fig. \ref{SIfig:3}} plots the wall‐clock time required to compute curvature on these Chung–Lu hypergraphs as the number of hyperedges $m$, the number of nodes $n$, and the average hyperedge size $\bar{d_v}$ varied one at a time. In all cases, HLRC's and HFRC's runtime keep around $0s$, while HORC's exhibits a steep growth.

\begin{figure}[ht]
    \centering
    \includegraphics[width=\linewidth]{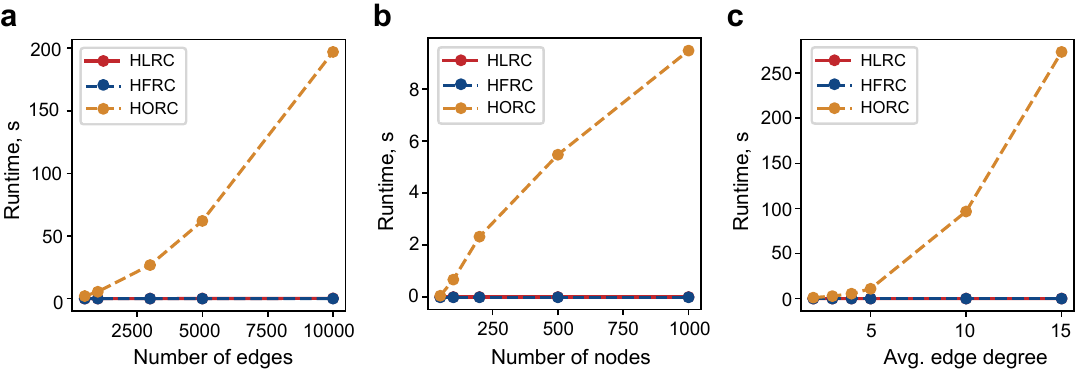}
    \caption{ \textbf{Runtime comparison under varying hypergraph parameters.} 
    \textbf{a}. The number of hyperedges $m$ varies over {500,1000,3000,5000,10000}. \textbf{b}. The number of nodes $n$ varies over {50,100,200,500,1000}. \textbf{c}. The average edge degree $\bar{d_v}$ varies over {2,3,4,5,10,15}. In each plot, only the indicated parameter is changed while the other two remain fixed at their baseline values ($m$=1000, $n$=500, $\bar{d_v}$=4).}
    \label{SIfig:3}
\end{figure}

\subsection{Detailed dataset descriptions}

The contact high school dataset\cite{Mastrandrea-2015-contact} was gathered over several days in December 2013 at a French high school, where 327 students and staff each wore an RFID badge that recorded proximity—any two badges within about 1–1.5 m exchanged signals every 20 seconds. From each 20-second snapshot of pairwise contacts, a proximity graph is constructed (nodes = people; edges = detected interactions), and its maximal cliques—groups of individuals all mutually in contact—are extracted as timestamped hyperedges. In total, the dynamic hypergraph comprises 172035 timestamped hyperedges (7937 unique) ranging from dyads up to gatherings of five. For a static representation, the 7818 distinct cliques observed across all intervals were retained. In addition, each participants belonged to nine second-year classes: three in Biology (2BIO1, 2BIO2, 2BIO3), three in Mathematics–Physics(MP, MP*1, MP*2), two in Physics–Chemistry (PC and PC*) and one in Engineering Sciences (PSI*).

The Multi‐Attribute dataset (MADStat) on Statisticians aggregates bibliographic records from 36 leading statistics journals spanning 1975–2015, as compiled and cleaned by \cite{ji2022co}. It comprises 83331 papers and 47311 distinct authors. To build its hypergraph, each author is a node, and every set of co‐authors on a single paper forms one hyperedge. Hyperedges carry metadata for publication year and journal, enabling temporal and venue‐specific analyses. The dataset also includes author names and paper titles, enabling more detailed analyses.

The MAG-10\cite{amburg2020clustering, Sinha-2015-MAG} hypergraph is drawn from a cleaned subset of the Microsoft Academic Graph focused on ten flagship computer‐science conferences (for example, WWW, KDD, ICML, and so on).  In this construction, each node represents an individual author, and each hyperedge captures the full set of co‐authors on a single paper presented at one of those ten venues. To ensure consistency, any paper with more than 25 authors was excluded, and if the same group of authors appeared at multiple conferences their most frequent venue determined the hyperedge’s categorical label (ties were discarded).  The resulting static hypergraph comprises 80198 author-nodes and 51888 publication-hyperedges, each labeled by one of the ten conference categories; hyperedge sizes range from small collaborations (pairs or triples) up to the 25‐author cap. In MAG-10, author names and paper titles are not provided.

The Stex collection collected by \cite{coupette2023ollivier} comprises a suite of hypergraphs, each drawn from a different StackExchange community and built directly from its tagging system. In every Stex hypergraph, nodes are the distinct tags used on that site, and each hyperedge corresponds to one question, connecting the set of tags applied to it.  Across the 36 communities we examined (Table 8–11 in \cite{coupette2023ollivier}), the number of tags per site ranges from a few dozen (e.g. “tex” with 2035 tags) to several thousand (e.g. “superuser” with 5676 tags), while the number of questions spans from a few thousand (e.g. “japanese” at 26365) up to nearly half a million (e.g. “superuser” at 480854). Typical average hyperedge size (tags per question) lies between 2.0 and 3.0. For the task of hypergraph clustering, the three groups of hypergraphs were selected as the same as\cite{coupette2023ollivier}.

In the Mus collection\cite{coupette2023ollivier}, each hypergraph encodes a single musical piece drawn from the open-source music21 corpus. Using music21’s symbolic notation, every distinct pitch (sound frequency) was represented in a composition as a node. Hyperedges capture chords—i.e. all pitches sounding simultaneously at a given offset and for a specified duration. They restrict the collection to polyphonic works by a set of classical composers—Bach, Beethoven, Chopin, Haydn, Handel, Monteverdi, Mozart, Palestrina, Schumann, Schubert, Verdi, Joplin, Trecento, and Weber—and exclude primarily monophonic pieces.  As shown in Table 4 of \cite{coupette2023ollivier}, each selected hypergraph is summarized by its number of nodes $n$, number of edges $m$, and the distribution of hyperedge sizes (from 0 up to 12). For the clustering experiments, composers were grouped strictly by historical era. Giovanni Palestrina’s 1318 works—representing 67.4\% of all hyperedges—were omitted to prevent a single author from dominating the results. The remaining composers were then divided into three balanced, era‐based clusters: "Renaissance" (Monteverdi and Trecento), "Baroque" (Bach and Handel), and "Later" (Haydn, Mozart, Beethoven, Schumann, Schubert, Chopin, Verdi, Weber, and Joplin).

\subsection{Expanded results.}

For the MADStat dataset, \textbf{Supplementary Fig. \ref{SIfig:4}}a–c visualizes the one‐hop subgraphs surrounding the three hyperedges with the most negative HLRC values. Each of these hyperedges links two exceptionally well‐connected statisticians whose respective collaboration networks form large, disjoint clusters with no overlapping co‐authors—precisely the “bottleneck” structures that negative curvature highlights. Tracing the author names confirms they are among the most influential researchers in statistics, further validating HLRC’s ability to detect bridge‐like collaborations.

We performed a similar analysis on the MAG-10 co‐authorship hypergraph, which collects co‐author sets from ten major computer-science conferences. \textbf{Supplementary Fig.} \ref{SIfig:4}d shows the distribution of HLRC for each conference. \textit{The ACM Symposium on Theory of Computing} (STOC)—the flagship venue for foundational theoretical computer‐science research—exhibits the lowest average curvature, while \textit{the International World Wide Web Conference} (WWW)—a forum centered on web architecture, search, information retrieval, and related applied topics—has the highest average curvature. This mirrors the MADStat findings: more theory-oriented venues tend toward negative curvature (indicative of bridge‐like structures), whereas more application-driven venues show positive curvature (reflecting tighter, community-like groups). Finally, \textbf{Supplementary Fig. \ref{SIfig:4}}e depicts the one‐hop neighborhood around the single most negative hyperedge in MAG-10: it connects two authors who each have large, non‐overlapping collaborator sets (127 and 90 neighbors, respectively), again illustrating HLRC’s sensitivity to sparse‐overlap, high‐degree bridges. Because MAG-10 does not include author names, we cannot identify the individuals involved.

\begin{figure}[ht]
    \centering
    \includegraphics[width=\linewidth]{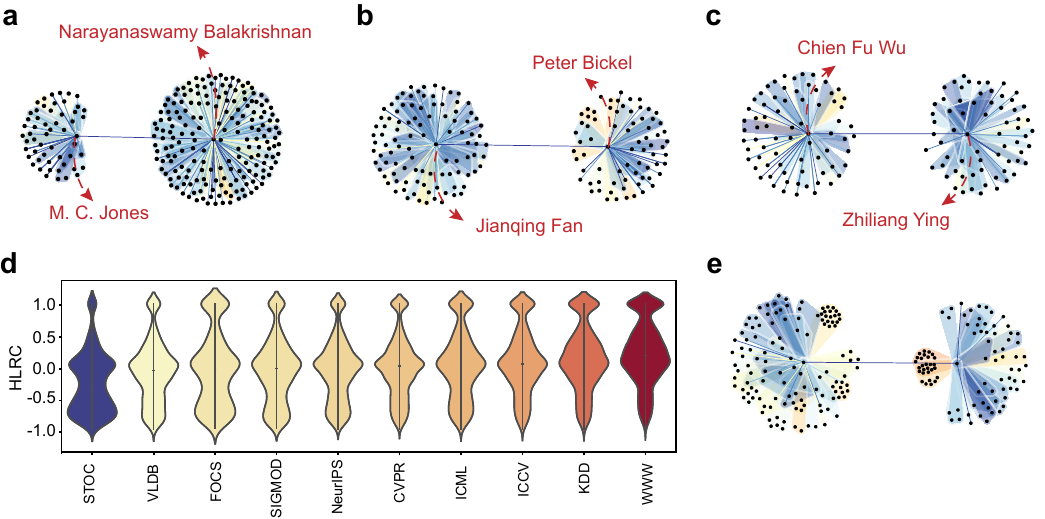}
    \caption{\textbf{Hyperedges with extreme low HLRC values and HLRC distribution across conferences.} \textbf{a-c}. One-hop subgraphs around the three hyperedges with the most negative HLRC in MADStat co-authorship network. 
    \textbf{a}. The lowest HLRC hyperedge (–0.976) links N. Balakrishnan and M.C. Jones with 186 and 55 collaborators respectively and no shared neighbors;
    \textbf{b}. The lowest HLRC hyperedge (–0.976) links Jianqing Fan and Peter Bickel with 101 and 73 collaborators respectively and no shared neighbors.
    \textbf{c}. The lowest HLRC hyperedge (–0.971) links Chien Fu Wu and Zhiliang Ying with 73 and 67 collaborators respectively, again with no common collaborators.
    \textbf{d}. HLRC distributions for MAG-10 hyperedges grouped by conference type, illustrating that theoretical venues exhibit more negative HLRC values and applied venues more positive.
    \textbf{e}. One-hop subgraph around the most negative hyperedge in MAG-10 ($\mathrm{HLRC}(e)$ = –0.981), linking authors with 127 and 90 neighbors and no shared collaborators; because MAG‐10 does not include author names, the individuals remain unidentified.}
    \label{SIfig:4}
\end{figure}

\subsection{AMI and ARI}
When comparing a proposed clustering to a known “ground-truth” partition, it is essential to use measures that both correct for chance agreement and are insensitive to the absolute number or sizes of clusters. Two widely used external validation metrics are the Adjusted Rand Index (ARI\cite{hubert1985comparing}) and the Adjusted Mutual Information (AMI\cite{vinh2009information}). Both begin with a raw similarity score—pairwise agreement for ARI, mutual information for AMI—and then subtract the expected value under random labeling before normalizing, so that perfect recovery scores 1 and chance agreement scores 0.

Let the ground-truth partition  be $X=\{X_1,...,X_K\}$ and the clustering under evaluation be $Y=\{Y_1,...,Y_K\}$, each dividing the same $n$ points into $K$ groups. Define,
$$
n_{ij} = |X_i \cap Y_j|,\quad  a_i = \sum_{j=1}^K n_{ij},\quad  b_j = \sum_{i=1}^K n_{ij},\quad  \binom{n}{2} = \frac{n(n-1)}{2}.
$$
The unadjusted Rand index counts the fraction of point-pairs on which the two partitions agree-either both assigning them to the same group or both keep then apart. The ARI refines this subtracting the expected number of agreeing pairs under a random model with fixed $\{a_i\}$ and $\{b_j\}$ and then dividing by the maximal possible excess above change:
$$
\mathrm{ARI} = \frac{\sum_{i,j}\binom{n_{ij}}{2} - \frac{\sum_i \binom{a_i}{2}\sum_j \binom{b_j}{2}}{\binom{n}{2}}}{\frac{1}{2}\left(\sum_i \binom{a_i}{2} + \sum_j \binom{b_j}{2}\right) - \frac{\sum_i \binom{a_i}{2}\sum_j \binom{b_j}{2}}{\binom{n}{2}}}.
$$
By construction, $\mathrm{ARI}=1$ when the partitions are identical, $\mathrm{ARI}=0$ if their agreement is  no better than random, and $\mathrm{ARI}<0$ when they agrees less than expected by chance.

In AMI, the two labelings were treated as discrete random variables $X$ and $Y$ over $\{1,...,K\}$. Their mutual information $$I(X;Y) = \sum_{i,j}\frac{n_ij}{n}\log\left(\frac{n_{ij}/n}{(a_i/n)(b_j/n)}\right)$$ measures how much knowing one labeling reduces uncertainty about the other. However, even independent partitions share some mutual information by chance. Denoting by $E[I(X;Y)]$ its expectation under random labelings with the same cluster-size profiles, the AMI is defined as $$\mathrm{AMI} = \frac{I(X;Y)-E[I(X;Y)]}{\frac{1}{2}(H(X)+H(Y))-E[I(X;Y)]},$$ where $$H(X) = -\sum_{i=1}^K\frac{a_i}{n}\log\left(\frac{a_i}{n}\right),\quad H(Y) = \sum_{i=1}^K\frac{b_i}{n}\log\left(\frac{b_i}{n}\right).$$This normalization ensures $\mathrm{AMI}=1$ for perfect correspondence and $\mathrm{AMI}=0$ for chance-level overlap; under the usual null model AMI never becomes negative.
\endgroup

\end{document}